\documentclass[letterpaper, 10 pt, conference]{ieeeconf}  

\IEEEoverridecommandlockouts                              

\usepackage{amsthm}
\usepackage{graphics} 
\usepackage{epsfig} 
\usepackage{mathptmx} 
\usepackage{times} 
\usepackage{amsmath} 
\usepackage{amssymb}  
\usepackage{amsbsy}
\usepackage{bm}
\usepackage{url}
\usepackage{hyperref}
\hypersetup{
    colorlinks=true,
    linkcolor=blue,
    filecolor=magenta,      
    urlcolor=cyan,
    pdftitle={Overleaf Example},
    pdfpagemode=FullScreen,
    }

\usepackage{xcolor}
\usepackage{algorithm}
\usepackage{algorithmic}
\usepackage{caption}
\usepackage{subcaption}
\usepackage{verbatim}
\usepackage{lipsum}
\usepackage[percent]{overpic}
\usepackage[export]{adjustbox}
\newtheorem{example}{Example}
\newtheorem{theorem}{Theorem}

\title{\LARGE \bf
A Recurrent Differentiable Engine for Modeling Tensegrity Robots\\ Trainable with Low-Frequency Data}

\author{Kun Wang, Mridul Aanjaneya and Kostas Bekris
\thanks{The authors are with the Department of Computer Science, Rutgers University, NJ 08901, USA. Email:
        {\tt\small {kun.wang2012, mridul.aanjaneya, kostas.bekris}@rutgers.edu}. This work has been partially supported by NSF award IIS 1956027, IIS-2132972, CCF-2110861 and the Rutgers University startup grant.}%
}

\begin{document}
\maketitle
\thispagestyle{empty}
\pagestyle{empty}


\begin{abstract}
Tensegrity robots, composed of rigid rods and flexible cables, are difficult to accurately model and control given the presence of complex dynamics and  high number of DoFs. Differentiable physics engines have been recently proposed as a data-driven approach for model identification of such complex robotic systems. These engines are often executed at a high-frequency to achieve accurate simulation. Ground truth trajectories for training differentiable engines, however, are not typically available at such high frequencies due to limitations of real-world sensors. The present work focuses on this frequency mismatch, which impacts the modeling accuracy. We proposed a recurrent structure for a differentiable physics engine of tensegrity robots, which can be trained effectively even with low-frequency trajectories.  To train this new recurrent engine in a robust way, this work introduces relative to prior work: (i) a new implicit integration scheme, (ii) a progressive training pipeline, and (iii) a differentiable collision checker. A model of NASA's icosahedron SUPERballBot on MuJoCo is used as the ground truth system to collect training data. Simulated experiments show that once the recurrent differentiable engine has been trained given the low-frequency trajectories from MuJoCo, it is able to match the behavior of MuJoCo's system. The criterion for success is whether a locomotion strategy learned using the differentiable engine can be transferred back to the ground-truth system and result in a similar motion. Notably, the amount of ground truth data needed to train the differentiable engine, such that the policy is transferable to the ground truth system, is 1\% of the data needed to train the policy directly on the ground-truth system.
\end{abstract}



\section{Introduction}

\emph{Tensegrity robots} are compliant systems composed of rigid  (rods) and flexible elements (cables) connected to form a lightweight deformable structure. Their adaptive and safe features motivate applications, such as manipulation~\cite{lessard2016bio}, locomotion~\cite{sabelhaus2018design}, morphing airfoil~\cite{chen2020design} and spacecraft lander design~\cite{bruce2014superball}. At the same time, they are difficult to accurately model and control due to the high number of DoFs and complex dynamics.  This work is motivated by these exciting robotic platforms and aims to provide scalable solutions for modeling them in a data-driven, but explainable, manner.

Learning effective control policies for tensegrity robots is challenging with model-free solutions since they require a large amount of training data and collecting trajectories from tensegrities is time-consuming, cumbersome and expensive. Thus, a better alternative is to tune a dynamical model, or a simulator, of the system to minimize the difference between trajectories predicted by the model relative to those executed by a robot (henceforth referred to as the \emph{ground-truth system}). This is a \emph{system identification} problem, which is necessary before employing a model-based controller.



\emph{Differentiable physics engines} for dynamic trajectory prediction have emerged as a promising data-driven tool for system identification as they allow for data-driven model inference using backpropagation. Some of them are built entirely on neural networks~\cite{NIPS2016_3147da8a,sanchez2018graph,NEURIPS2018_fd9dd764} and can model many different systems, but are data hungry since they have a large number of hidden variables. Other differentiable engines, similar to this work, are built given first principles for a specific physical system~\cite{de2018end,le2021differentiable,lutter2020differentiable,LutSilWatPet21}, which allows them to be more data-efficient as they need to learn fewer parameters. They are also more explainable since they inform about the physical properties of the underlying system.  

\begin{figure}[t]
\centering
\includegraphics[width=\linewidth]{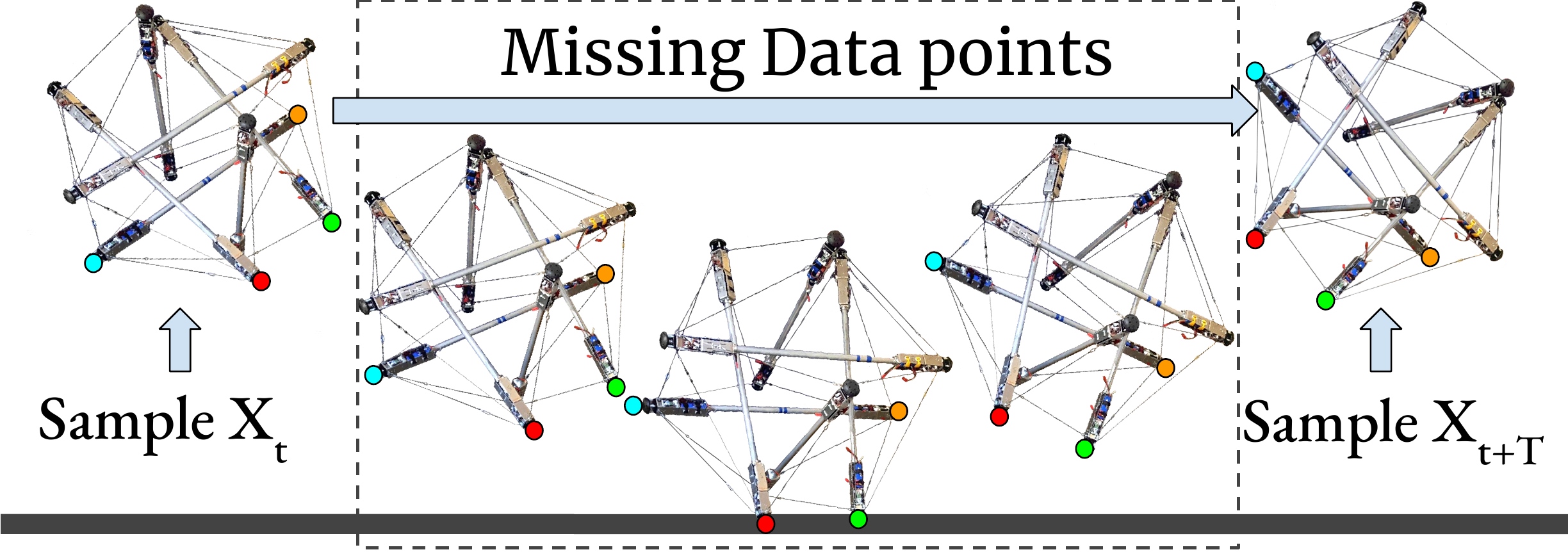}
\includegraphics[height=0.35\linewidth]{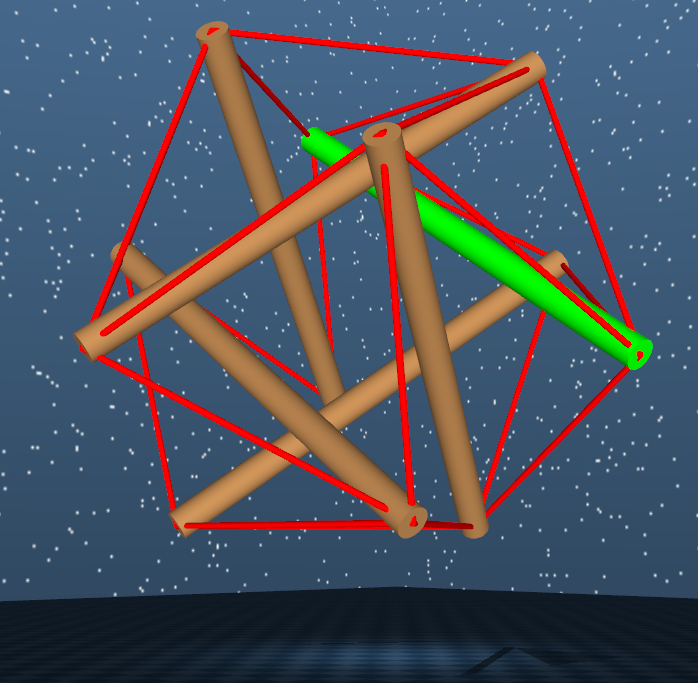}
\includegraphics[height=0.35\linewidth]{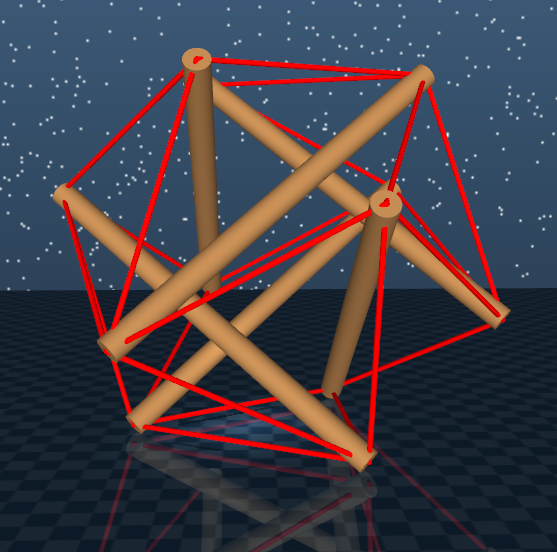}

\vspace{-.05in}
\caption{(top) A sparsely sampled trajectory, given observation frequency $T$, may skip critical states with contacts. (bottom-left) A non-contact setup used first in a progressive training process: the green rod of the tensegrity is kept fixed, while random forces are applied to the other rods. (bottom-right) A MuJoCo terrain environment where the robot is rolling. Training data comes from MuJoCo, which is also used for visualization.}
\label{fig:intro}
\vspace{-.3in}
\end{figure}

In previous work, the authors have developed the first differentiable engine targeted for tensegrity robots and argued its data efficiency \cite{wang2020end}. Nevertheless, critical gaps remain for the effective deployment of such tools on real robots. In particular, the data collection process for real robots commonly utilizes sensors, such as a camera or inertial measurement unit (IMU), where the frequency of the ground-truth data is constrained by the sampling rate of the sensor. In contrast, physics engines are executed at high frequencies for accurate simulation, especially for modeling contacts. This frequency discrepancy raises the challenge that the dynamics model needs to \emph{predict} the missing data points of a discretely sampled sparse trajectory. In addition, the missing data points may include events that are critical to the motion, such as collisions (see Fig.~\ref{fig:intro}(top)). The previous differential engine for tensegrities \cite{wang2020end} used a feed-forward architecture and can only be trained from trajectories sampled at high frequencies, similar to that of the engine's frequency, i.e., in the order of 1000Hz, which beyonds what most sensors can achieve. 

Motivated by this frequency gap, this work proposes a recurrent architecture for a differentiable physics engine targeted for tensegrity robots, which remains explainable and can be trained  effectively even with low-frequency ground-truth data. Beyond the recurrent nature of the engine, this work introduces a new implicit integration and an adaptive semi-implicit integration process as part of the differentiable engine to increase its robustness. Furthermore, a progressive training algorithm is proposed, which avoids gradient explosion and leads to a fast and stable training process. To further accelerate the training process, a fast differentiable collision checker has been implemented as part of the proposed engine. The authors will release the integrated solution as an open-source software package upon acceptance. 


To demonstrate the benefits of the proposed engine, sim2sim experiments are performed as a step towards the application on real robots. In particular, a model of NASA’s icosahedron SUPERballBot on MuJoCo is used as the ground truth system to collect training data. The experiments show that the recurrent differentable engine can match the output of the ground truth system even if it is provided low-frequency data about the MuJoCo trajectories. This paper demonstrates that a control strategy trained on the proposed recurrent engine after identification can be transferred back to the ground truth system and results in a similar motion. Notably, the amount of ground truth data needed to identify the recurrent engine is 100$\times$ less than the data needed to train the policy directly on the ground-truth system. 


\section{Related Work}

\emph{Differentiable physics} is an active area of research.  A hybrid engine with 2-way coupling for both rigid and soft bodies with smoothed frictional contact was proposed in~\cite{geilinger2020add}. A sphere and plane-based differentiable engine has been used for training a controller recurrently but without system identification~\cite{degrave2019differentiable}. A differentiable programming language was proposed in~\cite{hu2019difftaichi}, but requires users to implement the physics engine themselves. To reduce data requirements, analytical differentiable physics engines based on linear complementarity (LCP)~\cite{de2018end} or quadratic programming (QP)~\cite{le2021differentiable} have been proposed. However, they require densely sampled simulation trajectories. Interactions that can be efficiently simulated and differentiated have been studied~\cite{qiao2020scalable}, but interactions represented by neural networks~\cite{sanchez2018graph,hwangbo2019learning,NEURIPS2018_fd9dd764} have been shown to perform poorly for tensegrity robots~\cite{pmlr-v120-wang20b}. The Constrained Recursive Newton-Euler Algorithm (RNEA) ~\cite{lutter2020differentiable,LutSilWatPet21} uses differentiable physics engines for rigid-body chain manipulators~\cite{sutanto2020encoding, sanchez2018graph}. Differentiable dynamical constraints can also be applied to path optimization~\cite{toussaint2018differentiable}, training by image supervision~\cite{murthy2020gradsim,heiden2019interactive}, soft object manipulation~\cite{heiden2021disect}, soft robot control~\cite{9392257,bacher2021design} and quantum molecular control~\cite{wang2020differentiable}. A differentiable engine specifically designed for tensegrity robots~\cite{wang2020end} provides analytically explainable models for both the robot and the ground, but still requires densely sampled ground-truth trajectories for system identification.

\emph{Recurrent structures} use previous outputs as inputs while maintaining hidden states. This feature allows them to be trained on sequences to reflect the underlying temporal dynamics. Thus, recurrent structures have become popular solutions for video prediction~\cite{wu2021motionrnn}, embedded dynamics~\cite{suzuki2021compensation}, trajectory forecasting~\cite{Kothari_2021_CVPR} and translation~\cite{guo2018hierarchical}, etc.

Prior work on \emph{tensegrity locomotion}~\cite{zhang2017deep,luo2018tensegrity,surovik2019adaptive} has achieved complex behaviors, on uneven terrain, using the NTRT simulator~\cite{NTRTSim}, which was manually tuned to match a real platform~\cite{mirletz2015towards,caluwaerts2014design}. Many of these approaches use reinforcement learning to learn policies given sparse inputs, which can be provided by onboard sensors~\cite{luo2018tensegrity} and aim to address the data requirements of RL approach~\cite{surovik2019adaptive}, including by training in simulation. But simulated locomotion is hard to replicate on a real platform, even after hand-tuning, which emphasizes the importance of learning a transferable policy that is the focus of this work.
\emph{Domain randomization} ~\cite{chebotar2019closing,truong2021bi,mehta2020active} is a possible way to close the sim2real gap. Previous domain randomization efforts~\cite{chebotar2019closing} often assume the system parameters are within a Gaussian distribution and use a non-differentiable physics engine. Nevertheless, the parameter distribution here is unknown and the range varies a lot. The current work is a step in mitigating the reality gap by showing sim2sim transfer with low-frequency data. The different, closed-source and unknown physical model of the \emph{MuJoCo} simulator~\cite{todorov2012mujoco} is used as the ground-truth. MuJoCo does not follow the same first principles for modeling contacts as that of the proposed differentiable physics engine. This feature makes MuJoCo a good candidate for a ground-truth system.

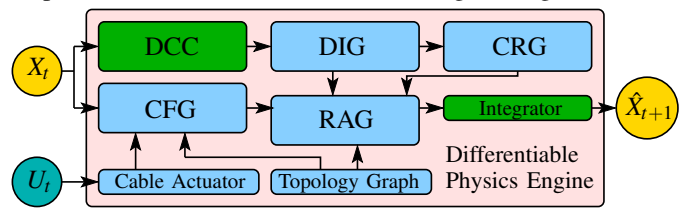
\begin{figure}[b]
\vspace{-.25in}
\centering
\begin{overpic}[width=0.5\textwidth]{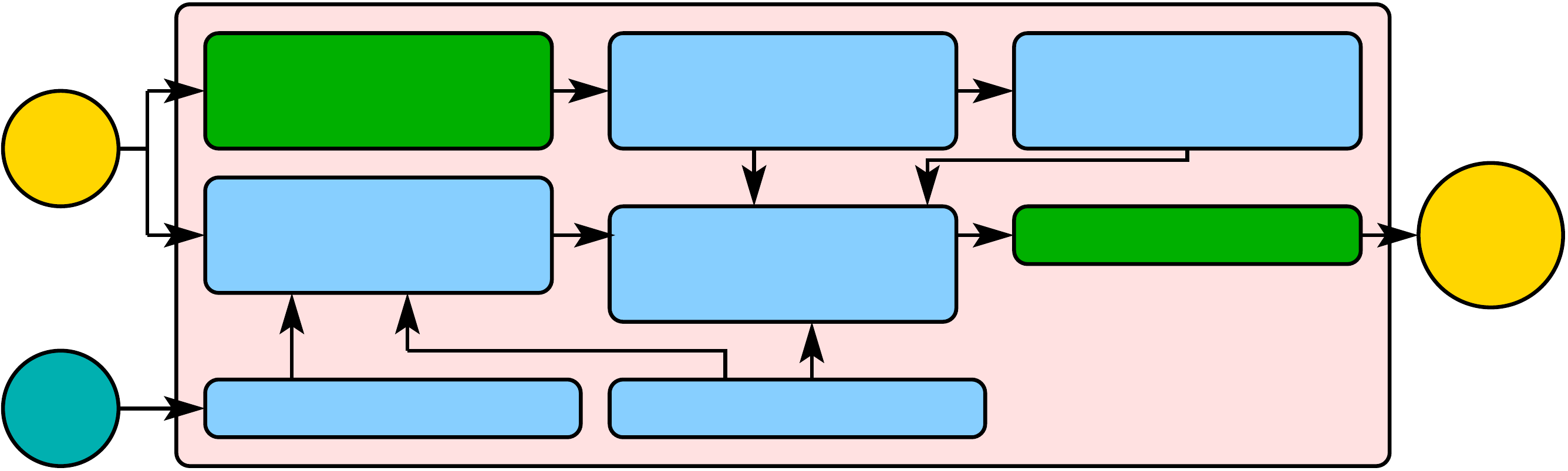}
    \put(2.5,20){$X_t$}
    \put(20,23){DCC}
    \put(72,23){CRG}
    \put(20,13.5){CFG}
    \put(46,23){DIG}
    \put(2.5,3){$U_t$}
    \put(15.5,3){\footnotesize Cable Actuator}
    \put(40,3){\footnotesize Topology Graph}
    \put(65,7){\small Differentiable}
    \put(65,3){\small Physics Engine}
    \put(92,14){$\hat{X}_{t+1}$}
    \put(70,14){\footnotesize Integrator}
    \put(46,12){RAG}
\end{overpic}
\caption{The physics engine takes the current robot state $X_t$ and control $U_t$ as inputs and predicts the next state $\hat{X}_{t+1}$. Compared to previous attempts~\cite{wang2020end}, this work introduces recurrent training (as shown in Fig. \ref{fig:recurent_engine}), a new numerical integrator, a progressive training pipeline and a new Differentiable Collision Checker (DCC) to account for the frequency mismatch. DIG: Dynamic Interaction Graph, CRG: Collision Response Generator, CFG: Cable Force Generator, RAG: Rod Acceleration Generator.}
\label{fig:single_step_engine}
\end{figure}

\vspace{-1mm}
\section{Proposed Engine}
\vspace{-1mm}
At the core of the proposed approach lies a differentiable physics engine, shown in Fig.~\ref{fig:single_step_engine}, which builds on top of previous work ~\cite{wang2020end}. The engine brings together a series of analytical models based on first principles, which are linear and differentiable. The input to the engine is the current robot state $X_t$ and the instantaneous control $U_t$. The engine internally stores a representation of the robot in a static topology graph indicating the connectivity of rods and cables. The control $U_t$ is passed to a Cable Actuator module, which maps the control to desired cable rest-lengths. Together with the topological graph, the cable actuator informs the Cable Force Generator (CFG), which is responsible to predict the forces applied on the rods due to the cables given the latest robot state. In parallel and given the robot state, a Differentiable Collision Checker (DCC) detects collisions and informs a Dynamic Interaction Graph (DIG), which stores the colliding bodies (either rod-to-rod or rod-to-ground) and the corresponding contact points. This information is passed to a Collision Response Generator (CRG), which is responsible to compute the reaction forces applied to the rods. The forces and torques from the cables (as computed by CFG) and those from contacts (as computed by CRG) are forwarded to the Rod Acceleration Generator (RAG), which computes the linear and angular acceleration experienced by the rods. These accelerations and torques are integrated by an Integrator module so as to update the robot state. 

Overall, these models introduce two sets of parameters: robot parameters and contact parameters, which need to be identified. The new contributions of this work include: (i) a recurrent architecture for training this engine with low frequency data; (ii) an implicit integration scheme, (iii) a progressive training pipeline with a new adaptive semi-implicit integration module, and (iv) a differentiable collision checker. These components enable stable and robust recurrent training with low-frequency training data on a high-frequency physics engine.
\vspace{-1mm}
\subsection{Recurrent Differentiable Physics Engine}
\vspace{-1mm}
At a high level, the proposed training pipeline calls the differentiable physics engine in a recurrent fashion through a sequence of temporal connections. With help from Back-Propagation Through Time (BPTT), the engine parameters can be identified given low-frequency sampled trajectories. Figure~\ref{fig:recurent_engine} presents the recurrent nature of the proposed architecture. The trajectory is split to data point tuples, $[(X_t, \bar{U_t}), X_{t+T}]$, where $X_t$ is the system state at time $t$ and $\bar{U_t}$ is the sequence of actions to be executed in the interval $[t,t+T]$. The control sequence $\bar{U_{t}}$ can be split into control signals $U_t,U_{t+1},\ldots,U_{t+T-1}$, which have the same high-frequency as the engine. In recurrent mode, the engine receives as inputs $(X_t, \bar{U_t})$ and generates output $\hat{X}_{t+T}$.  The loss is the mean square error (MSE) between the predicted state $\hat{X}_{t+T}$ and the ground truth state $X_{t+T}$. The gradients of the loss function are then back-propagated through the physics engine and used for updating the engine's parameters. 

\begin{figure}[t]
\centering
\begin{overpic}[width=\columnwidth]{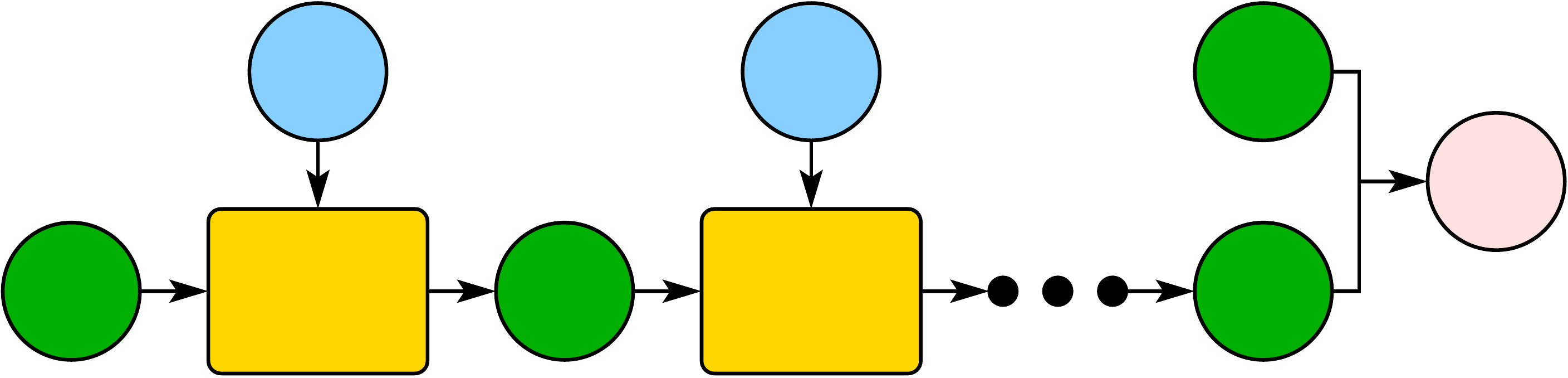}
    \put(3,4){$X_t$}
    \put(18.5,18){$U_t$}
    \put(16.5,4){DPE}
    \put(32.5,4){$\hat{X}_{t+1}$}
    \put(48,18){$U_{t+1}$}
    \put(48,4){DPE}
    \put(77,4){$\hat{X}_{t+T}$}
    \put(77,18){$X_{t+T}$}
    \put(92.5,11.5){\footnotesize Loss}
\end{overpic}
\caption{The recurrent training pipeline is trained via supervision given a state-action tuple $X_t, \bar{U_t}, X_{t+T}$. The engine will generate all missing states between $X_t$ and $X_{t+T}$. The mean square error (MSE) loss between the predicted state $\hat{X}_{t+T}$ and the ground truth $X_{t+T}$ generates the gradients needed to update the engine's internal parameters. DPE: Differentiable Physics Engine.}
\label{fig:recurent_engine}
\vspace{-.3in}
\end{figure}

To adapt this engine for recurrent training and achieve robustness, this work introduces an implicit integration scheme and a differentiable collision checker.
\vspace{-1mm}
\subsection{Implicit Integration}
\vspace{-1mm}
While implicit integration schemes are standard for simulating stiff multi-body systems, their dynamics are typically handled by explicit methods, such as semi-explicit integration~\cite{geilinger2020add}. Semi-explicit integration is not stable for stiff systems, such as tensegrities.  Deriving an implicit system, however, for the whole tensegrity is very complex. For instance, more complex than what has been achieved using a particle-based mesh system~\cite{qiao2020scalable}. To address this complexity, this work follows a modular approach by focusing on the basic elements of tensegrity robots, rods and springs, and derives the following $Ax=b$ form linear system for each time step,where $A$ is a 21x21 matrix and $x, b$ are 21x1 vectors (see appendix\footnote{\href{https://sites.google.com/view/recurrentengine}{https://sites.google.com/view/recurrentengine}} for details):
\begin{align*}
    \begin{bmatrix}
    0 & 1 &  0 & & & & \\
    0 &  0 & 1 & \multicolumn{4}{c}{-\boldsymbol{C}} \\
    -1 & K & k \boldsymbol{B} & & & & \\
      &  &  & 0 & 1 & 0 & 0 \\
    \multicolumn{3}{c}{-\boldsymbol{a}\boldsymbol{e}_1^T} & 1 & -\Delta t  & 0 & 0\\
      &  &  & 0 & 0 & 1 & 0 \\
      &  &  & 0 & 0 & -\Delta t & 1
    \end{bmatrix}
    \begin{bmatrix}
     f_{t+1} \\
     x_{t+1}^{m_1} \\
     v_{t+1}^{m_1} \\
     x_{t+1}^R \\
     v_{t+1}^R \\
     \omega_{t+1}^R \\
     r_{t+1}
    \end{bmatrix} = 
    \begin{bmatrix}
        0 \\
        0 \\
        c_1  \\
     v_{t}^R \\
     x_{t}^R \\
     \omega_{t}^R \\
     r_{t}
    \end{bmatrix} \\
    \boldsymbol{B}=\begin{bmatrix}
        \hat{\Delta x_{t}}_x (\hat{\Delta x_{t}}_x & \hat{\Delta x_{t}}_y & \hat{\Delta x_{t}}_z) \\
                \hat{\Delta x_{t}}_y (\hat{\Delta x_{t}}_x & \hat{\Delta x_{t}}_y & \hat{\Delta x_{t}}_z) \\
                \hat{\Delta x_{t}}_z (\hat{\Delta x_{t}}_x & \hat{\Delta x_{t}}_y & \hat{\Delta x_{t}}_z) 
    \end{bmatrix}, \: \boldsymbol{C}=\begin{bmatrix}
        1 & 0 & 0 & 1 \\
        0 & 1 & 0 & [\omega_t^R \times] \\
        0 & 0 & 0 & 0
    \end{bmatrix}
\end{align*}
where $x^{m_1}$ and $x^R$ are the positions of the left end-point of the spring (denoted as $m_1$) and the rod, $r$ is the torque arm from $m_1$ to rod's center of mass, $v^{m_1}$ is the linear velocity of $m_1$, $v^R$ is the rod linear velocity, $\omega^R$ is the rod angular velocity, $\hat{\Delta x}$ is the spring direction, $l^{rest}$ is the spring rest-length, $f$ is the spring force under Hooke's law, $K$ is the spring stiffness, $k$ is the spring damping coefficient, $m$ is the rod mass, $I^{-1}$ is the inverse of the rod inertia matrix, $[\omega_t^R \times]$ and $[r_t \times]$ are skew-symmetric matrices of $\omega_t^R$ and $r_t$, $c_1 = K (x_{t}^{m_2} + l_{t}^{rest}) + k (v_{t}^{m_2}  \cdot \hat{\Delta x_{t}}) * \hat{\Delta x_{t}}$ and $m_2$ is the other spring end-point, $\boldsymbol{a}=\Delta t(1/m,0,I_{t}^{-1} [r_{t} \times],0)^T$, and $\boldsymbol{e}_1=(1,0,0)^T$.

\vspace{-1mm}
\subsection{Progressive Training via Implicit/Semi-implicit Integr.}
\vspace{-1mm}
The objective of progressive training is to mitigate the frequency discrepancy between the training data and the physics engine. Instead of 2-way coupling ~\cite{geilinger2020add,shinar2008two}, which applies different integration methods for soft and rigid bodies, this work proposes a progressive training method for fast and stable convergence. The idea is to train the engine using implicit integration first in a time-stepping way to find a coarse grained model first, and then fine tune with semi-implicit integration for a more accurate model. The implicit integration has a large region of absolute stability in terms of parameters and time step but it is computationally expensive and reduces system energy steadily over time. Fine tuning with semi-implicit integration is fast and more accurate while it maintains system energy. The proposed progressive training approach takes advantage of both schemes' benefits. In order to achieve a correct gradient, a variation of semi-implicit integration is also proposed, named as adaptive naive integration (ANI), which is described in the appendix.

The progressive training algorithm is shown in Alg.~\ref{alg:progressive_training}, where $T$ is the sampling interval, $lr$ is the learning rate, and $\Delta t_r$ is the engine's simulation step. The engine is first trained recurrently with implicit integration (line 3-10), starting with a large time step and a large learning rate. Once the validation loss increases, the time step reduces gradually for a more precise simulation until the time step agrees with the physics engine's frequency (line 6-7). Once the training time step agrees with the frequency, the learning rate can reduce to stabilize the parameters and achieve a coarse grained model (line 8-9). Then, this model is fine-tuned with semi-implicit integration (line 12-20) and the learning rate keeps reducing (line 17-18), until a stable accurate model is achieved.

\begin{center}
\vspace{-.2in}
\begin{algorithm}
\caption{Progressive Training Algorithm}
\label{alg:progressive_training}
\textbf{Input}: Sparse Sampled Trajectory\\
\textbf{Parameter}: $T=100ms, lr=0.1, \Delta t_r=1ms$ \\
\textbf{Output}: Identified System Parameters\\
\vspace{-.2in}
\begin{algorithmic}[1] 
\STATE Let $\Delta t=T$.
\WHILE{$lr > 10^{-4}$}
\STATE Train engine recurrently with implicit integration.
\IF {validation loss is reducing}
\STATE continue
\ELSIF {$\Delta t > \Delta t_r$} 
\STATE $\Delta t = \max(\Delta t / 2, \Delta t_r)$
\ELSE 
\STATE $lr = lr / 2$
\ENDIF
\ENDWHILE
\STATE $lr=0.1, \Delta t = \Delta t_r$
\WHILE{$lr > 10^{-4}$}
\STATE Train recurrently with semi-implicit integration.
\IF{validation loss is reducing}
\STATE  continue
\ELSE
\STATE $lr = lr / 2$
\ENDIF
\ENDWHILE
\end{algorithmic}
\end{algorithm}
\vspace{-5mm}
\end{center}
\vspace{-1mm}
\subsection{Differentiable Collision Checker}
\vspace{-1mm}
The motion of the robot also depends upon the reaction forces and the friction between the robot and the ground. Thus, rich contacts should be modeled properly to simulate robot motions. In principle, the physics engine could use an off-the-shelf collision checker. The requirement for training the engine in a recurrent fashion and backpropagating the loss through the engine means that the collision checker should also be differentiable. An accompanying appendix provides an analysis of the gradients at contact points to show that regardless if the collision checker is differentiable or not, its output does not affect the computation of a correct gradient direction. A differentiable collision checker, however, can still accelerate the training process.

\vspace{-1mm}
\section{Experiments}
\vspace{-1mm}
The experiments use a model of the SUPERballBot tensegrity robot platform~\cite{Vespignani:2018:DSCT} simulated in the MuJoCo engine as the ground truth system. The state $X_t$ is 72-dimensional since there are 6 rigid rods for which the 3D position $\boldsymbol{p_t}$, linear velocity $\boldsymbol{v_t}$, orientation $\boldsymbol{q_t}$ and angular velocity $\boldsymbol{\omega_t}$ are tracked over time. The space of control signals $U_t$ is 24-dimensional, since there are 24 cables and it is possible to control the target length for each of them individually. 

The task is to estimate the robot's parameters, i.e., spring stiffness, damping, rod mass, and contact parameters, i.e., reaction and friction coefficients. The evaluation compares the predicted trajectories of the robot's center of the mass (CoM) by the differential engine against the ground truth CoM trajectory for the same controls. To generate control, this work evaluate multiple sample-based model predictive control (MPC) algorithms. The objective of the controller is to achieve a desired direction for the platform's CoM. 

The ground truth system on MuJoCo has been setup to mimic empirical motions of the real system at NASA Ames~\cite{Vespignani:2018:DSCT} and reflect the NTRT tensegrity robot simulator~\cite{NTRTSim}: rod mass $m$ is 10kg, length $2r$ is 1.684m the cable rest length $l^{rest}$ is 0.95m, the stiffness $K$ is 10,000 and damping $k$ is 1,000; the activation dynamics type is filter, the control range is $[-100, 100]$, and the gain parameter is 1,000. All sliding, torsional and rolling frictions are enabled with coefficient equal to 1.  A non-contact and a contact environment were setup as shown in Fig.~\ref{fig:intro}. For each environment, 10 trajectories were samppled from MuJoCo for training, 2 for validation and 10 for testing. All trajectories are 5 seconds long. The sampling frequency is 10Hz, however, the engine's frequency is 1000Hz.

\vspace{-1mm}
\subsection{Ablation Study}
\vspace{-1mm}

Ablation experiments evaluate the proposed components and show that: 1) recurrent training with implicit integration outperforms feed forward training with semi-explicit integration~\cite{wang2020end}; 2) progressive training outperforms implicit integration only training; 3) differentiable collision detection is faster than non-differentiable one.

\subsubsection{Necessity of Recurrent Training}
A naive idea to mitigate the frequency gap between simulation and sensing updates is to reduce the simulation frequency. In the considered setup, the simulation time step would increase from 1ms to 100ms, same to the sampling interval of ground truth trajectories, so as to avoid recurrent training. Fig.~\ref{fig:intro} (top) shows that this may miss important collision states. Even without collisions, this is still problematic. For instance, consider a non-contact environment where the robot is held in the air and executes random controls, as in Fig.~\ref{fig:intro} (bottom-left). we compare: a) training the differential engine in a feed-forward fashion with a fixed 100ms time steps and b) using the time-stepping recurrent training (lines 1-11 of Alg.~\ref{alg:progressive_training}). Because of the large time step, both of them apply implicit integration for stable simulation. Fig.~\ref{fig:100ms_vs_time_stepping} (left) shows that the proposed recurrent training outperforms feed-forward training~\cite{wang2020end}, since its mean square error (MSE) over the CoM prediction is much lower. 

\begin{figure}[t]
\centering
\begin{subfigure}[b]{0.49\textwidth}
\centering
\includegraphics[width=0.47\linewidth]{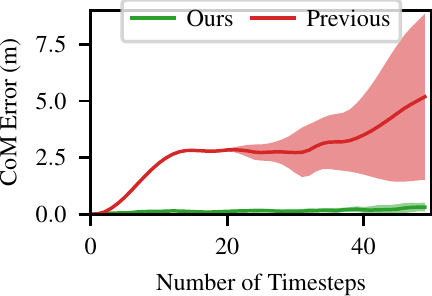}
\includegraphics[width=0.47\linewidth]{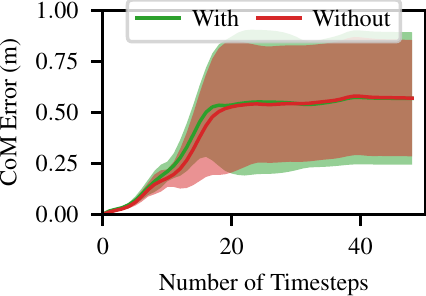}
\vspace{-4mm}
\caption{}
\label{fig:100ms_vs_time_stepping}
\end{subfigure}
\hfill
\begin{subfigure}[b]{0.49\textwidth}
\centering
\includegraphics[width=0.47\linewidth]{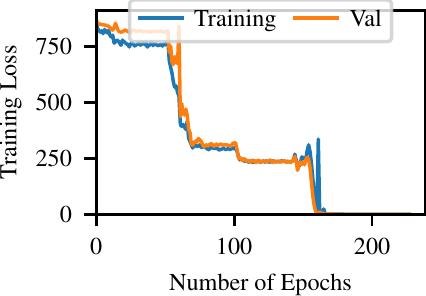}
\includegraphics[width=0.49\linewidth]{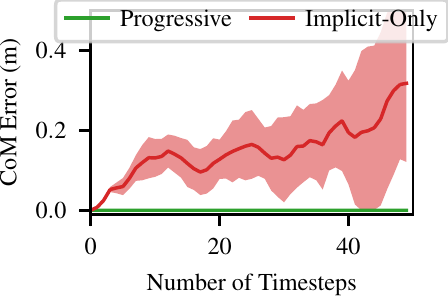}
\vspace{-4mm}
\caption{}
\vspace{-2mm}
\label{fig:implicit_vs_progressive}
\end{subfigure}
\caption{(a) Left: A previously proposed feed-forward  engine~\cite{wang2020end} exhibits large errors when trained over sparsely sampled trajectories (red line). The proposed recurrent process achieves significantly lower error (green line). Right: In the rolling environment, the predicted trajectory has similar performance w. or w/o the diff. collision checker. (b) Left: With implicit integration, the training loss stops decreasing after 100 epochs. After switching to semi-implicit integration at 151 epoch, the training loss drops sharply. Right: Progressive training results in a CoM error of magnitude 1e-5, which significantly outperforms implicit-only integration. The rod length is 1.682m. The average trajectory length is 2.57m. The error cloud shows the standard derivation.}
\vspace{-7mm}
\end{figure}

\subsubsection{Improvement of Progressive Training}

Previous efforts~\cite{qiao2020scalable,geilinger2020add} claimed good performance with implicit-integration only. This section shows that the proposed combination of implicit and semi-implicit integration significantly improves performance. The same task as in the above section above, where the process first trains with implicit-integration only and then fine-tunes with semi-implicit integration.

The comparison of implicit-only and progressive training is shown in Figure~\ref{fig:implicit_vs_progressive}. The first 150 epochs are from implicit-only training. The training loss keeps going down as the simulation step $\Delta t$ gets smaller. After 100 epochs even for $\Delta t = 1ms$, the loss is still very high. After switching to semi-implicit integration at epoch 151, the loss curve drops sharply since the semi-implicit Euler method conserves energy while the implicit method reduces it, which adds damping. Even though the implicit-only training cannot converge to the best parameters, it converges to a coarse solution good enough for stable training with semi-implicit integration.

\subsubsection{Speed up due to the Diff. Collision Checker}

Since a sparsely sampled trajectory may skip important contact states, the recurrent physics engine needs a collision checker that compensates for such information. But whether this collision checker should also be differentiable needs verification. This experiment trains the engine with a differentiable collision checker and, for comparison, with MuJoCo's non-differentiable collision checker. The trajectories involve the robot rolling on the ground as in Fig.~\ref{fig:intro} (right) along random directions and with random speed. The rolling trajectory includes multiple collisions between the rod and the ground.


Fig~\ref{fig:100ms_vs_time_stepping} (right) shows that for the same training setting, the engine trained with a differentiable collision checker and the engine trained with a traditional one have the same performance. This also means that the proposed engine can benefit from any off-the-shelf collision checker. This can save work from developing collision primitives for new objects. Nevertheless, the drawback of the traditional collision checker relates to training speed. Training with the differentiable collision checker on GPUs ($126\pm7s$ per epoch) is about 10x faster than using MuJoCo's collision checker on CPUs ($1296\pm14s$ per epoch). Although the experiment used a pool with 10 MuJoCo instances in parallel, training with a MuJoCo collision checker was significantly slower. The collision checking models rods as capsules, which is a computationally efficient primitive.  

\vspace{-1mm}
\subsection{System Identification with Sparse Trajectories}
\vspace{-1mm}

\begin{figure}[t]
\centering
\begin{subfigure}[b]{0.49\linewidth}
\centering
\includegraphics[width=\linewidth]{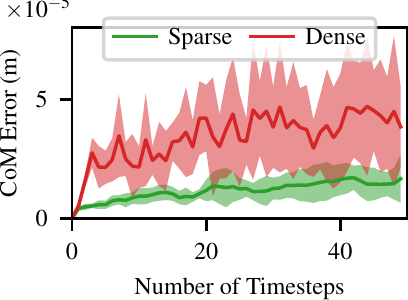}
\end{subfigure}
\begin{subfigure}[b]{0.49\linewidth}
\centering
\includegraphics[width=\linewidth]{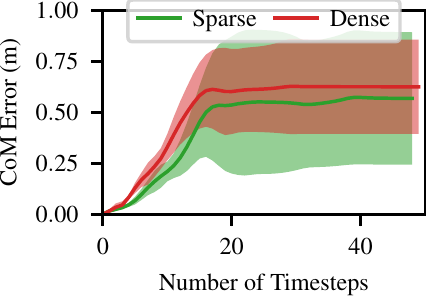}
\end{subfigure}
\vspace{-.2in}
\caption{(left) The proposed engine achieves smaller CoM position error in non-contact environments when trained on sparse trajectories, which contain only 1\% of data points of dense trajectories, relative to previous work~\cite{wang2020end} trained on dense trajectories. (right) Even when using only 1\% of data points, the proposed engine achieves comparable error regarding the CoM position in the contact environment. }
\label{fig:sparse_vs_dense}
\vspace{-8mm}
\end{figure}

This section compares the proposed training using sparsely sampled trajectories against the alternative~\cite{wang2020end} is trained with dense trajectories. The proposed solution achieves better performance at 100 times smaller sampling frequency, which amounts to only 1\% of data. 10 trajectories were sampled from MuJoCo for training, 2 for validation and 10 for testing. All trajectories are 5 seconds long. For the sparse trajectory dataset, data points $X_t$ are sampled every 100ms, then each trajectory has 50 data points. For the dense trajectory dataset, data points are sampled every 1ms, which is the same as the simulation step, then each trajectory has 5,000 data points, i.e., the sparse trajectory contains only 1\% of data points of the dense trajectory. The evaluation is performed separately for non-contact trajectories and contact trajectories.


Fig.~\ref{fig:sparse_vs_dense} shows that the proposed engine with sparse data can achieve similar and even better CoM error. The recurrent training process can average out the noise over smaller simulation steps, resulting to a more robust and accurate model. The prediction for non-contact trajectories has very low error with magnitude 10e-5, and on trajectories with contacts has error approx. 0.58m after 5,000 time steps. For the contact environment, the CoM error is divided by trajectory length to get CoM relative error. Considering the average trajectory length is 2.35m, the relative error is in the order of 24.7\% of its length. The robot starts with a random initial speed and stops in a range between 100 to 2,000 time steps. After approx. 2,000 time steps, the CoM error curve becomes flat because the ball stops rolling. The CoM error arises from the simplification of the contact model. This simplification reduces the need for training data, but also impacts identification accuracy, which is a trade-off between data requirements and model complexity. The error with contacts appears high but the objective is to achieve a data-efficient process for finding an explainable model for the target system that is accurate enough to train control policies that can be transferred back to the ground truth system. The following sections demonstrate this success criterion.
\vspace{-1mm}
\subsection{Sim2Sim MPC Policy Transfer}
\vspace{-1mm}
The goal is to train effective control policies in the engine that can be transferred to the original system painlessly. This work demonstrates sim2sim transfer given the MuJoCo simulator as a ground-truth system. Trajectories are sampled from it, and the proposed physics engine is used to identify critical parameters using these input trajectories. Then a policy is trained on the identified engine. The same policy generation is applied on MuJoCo and the comparison evaluates if the policies agree when executed on MuJoCo and how much training data are needed in either case. 

\begin{figure}[t]
    \centering
    \includegraphics[width=0.9\linewidth]{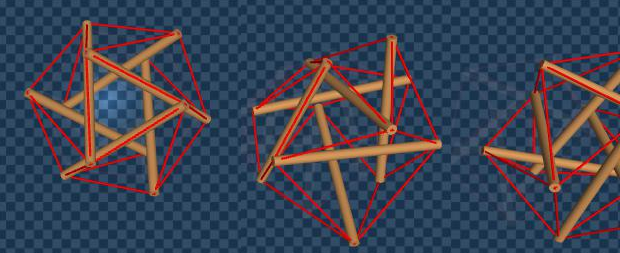}
    \vspace{-.05in}
    \caption{The SUPERball rolling task is to control the SUPERBall bot so as to move along the X axis (to the right) with speed $1m/s$.}
    \label{fig:superball_rolling_task}
    \vspace{-.25in}
\end{figure}

\begin{figure}[h]
\vspace{-.1in}
    \centering
    \includegraphics[width=0.9\linewidth]{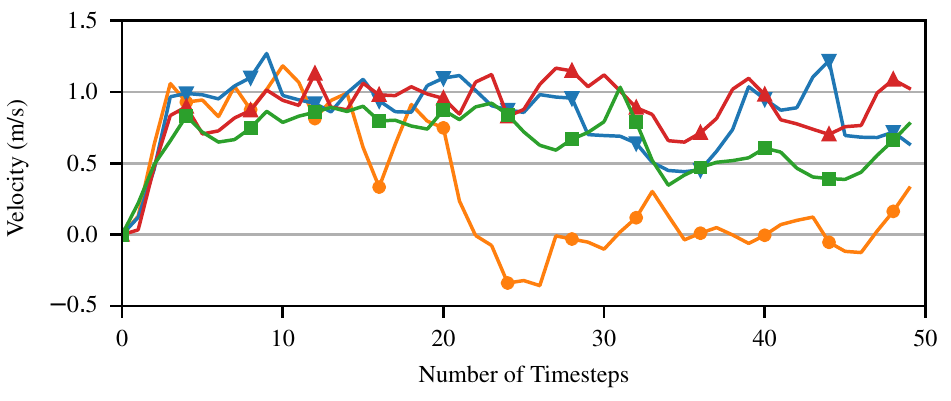}
    \includegraphics[width=0.9\linewidth]{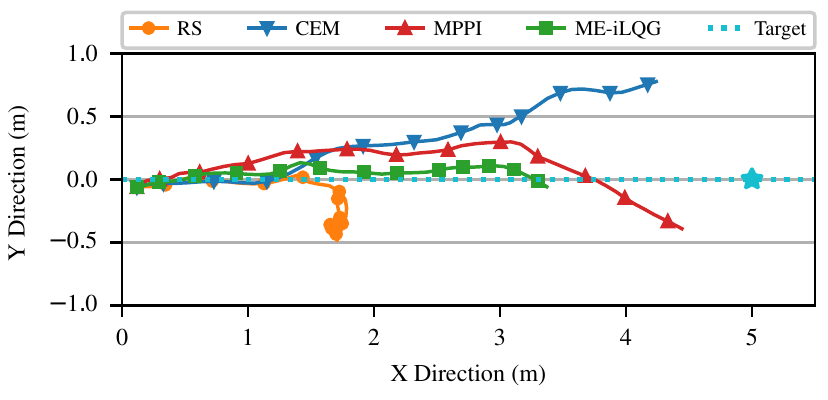}
    \vspace{-.1in}
    \caption{MPC algorithms evaluation on SUPERball bot. The task is to generate controls for 50 time steps (5 sec.) to roll the robot along the X direction at a target velocity of $1m/s$. The closer to the target velocity the better. (top) MPPI and CEM are most close to the target velocity. (bottom) The robot starts at $(0, 0)$ and the target position is at $(5, 0)$.  The trajectories of MPPI and CEM are closer along the X axis and end close to the target.}
    \vspace{-.15in}
    \label{fig:mpc_comparison}
\end{figure}

The first step is to identify a controller that is a good choice for the tesnegrity robot. The SUPERball rolling task shown in Fig.~\ref{fig:superball_rolling_task} is used to evaluate sampling-based MPC algorithms in MuJoCo that use random shooting (RS)~\cite{brooks1958discussion}, cross-entropy method (CEM)~\cite{botev2013cross}, model predictive path integral (MPPI) control~\cite{williams2017model} or max entropy iLQG (ME-iLQG)~\cite{levine2014learning}. The task is to roll the robot along the x axis with a velocity of $1m/s$. Each policy has $50$ time steps, and each time step is $0.1s$. The start  is at $(0, 0)$ and the target is $(5, 0)$.  MPPI and CEM generate the best policies among all algorithms. MPPI is slightly better since its trajectory ends closer to the target position and is the one selected. Equipped with this information, the next step is to evaluate whether the MPPI variant of MPC can be trained on the diff. physics engine and then transferred successfully on MuJoCo.

\subsubsection{Random v.s Identified Parameters}
MPPI is first trained and executed on MuJoCo to get a ground truth policy. The policy transfer involves training MPPI on the diff. engine and then executing it on MuJoCo. This section evaluates if the identification process is necessary for the transfer. Two diff. physics engines are compared with different parameters for stiffness, damping, mass, friction, etc.: 1) random parameters; 2) parameters identified with the recurrent training process. The engines share the same robot topology and dimensionality. MPPI is trained on these engines to get 2 sets of policies, which are executed on MuJoCo to compare with the ground truth policy. Fig~\ref{fig:policy_transfer_comparison} shows that the trajectories of the policy trained with random parameters doesn't match the policy trained on MuJoCo, since the achieved velocity is almost zero and the robot is stuck at the origin. The policy trained on the identified engine, however, is very close to the ground truth policy and achieves a velocity close to $1m/s$.

\begin{figure}[t]
    \centering
    \includegraphics[trim={0 0 0 4mm}, clip, width=0.9\linewidth]{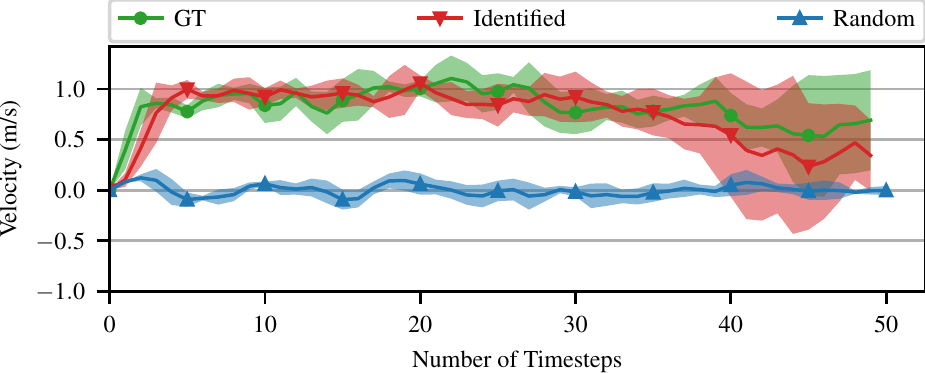}
    \includegraphics[width=0.9\linewidth]{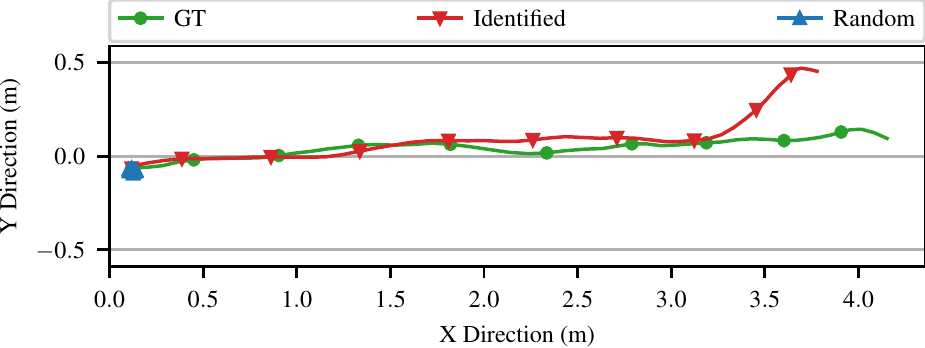}
    \vspace{-0.05in}
    \caption{Multiple trajectories are generated on MuJoCo from policies trained on the ground-truth (GT) system, the identified engine and an engine with random parameters. The policy learned on the identified engine has similar output (top:velocity,bottom:trajectory) on MuJoCo with the policy trained on MuJoCo, while the other policy does not result in motion of the CoM.}
    \label{fig:policy_transfer_comparison}
    \vspace{-.25in}
\end{figure}

\subsubsection{Data Requirement} Each policy has 50 controls. For each control, MPPI runs 5 iterations, samples 40 trajectories in each iteration and each trajectory is 1sec, i.e. 1,000 time steps. The sampling interval is 100ms. If the policy is directly trained on the ground truth system, then it requires: $50\times5\times40\times1,000/100=100K$ data points. The diff. engine only needs 10 5 sec. trajectories to identify non-contact parameters and 10 5sec. trajectories for contact ones, i.e. $(10+10)\times5,000/100=1000$ data points, which is only \textbf{1\%} of $100K$. It takes around 8 hours to train the recurrent engine and generate the control policy, which happens offline, and replaces the collection of real world data with compute.



\section{Conclusion}

This paper proposes a recurrent differentiable physics engine to identify parameters of tensegrity robots. This engine is data efficient, explainable and robust even with low-frequency training data. It can be used to train a controller and transfer it to the ground truth system without adaptation. The next step is to train the engine with sensing data for a real robot and use it to learn controllers, where uncertainty and partial observability must be also addressed.




\clearpage


\bibliographystyle{IEEEtrans}
\bibliography{references}

\newpage
\newpage
\section*{Appendix}
\section{Implicit Integration Derivation}
\begin{figure}[th]
\centering
\includegraphics[width=0.5\linewidth]{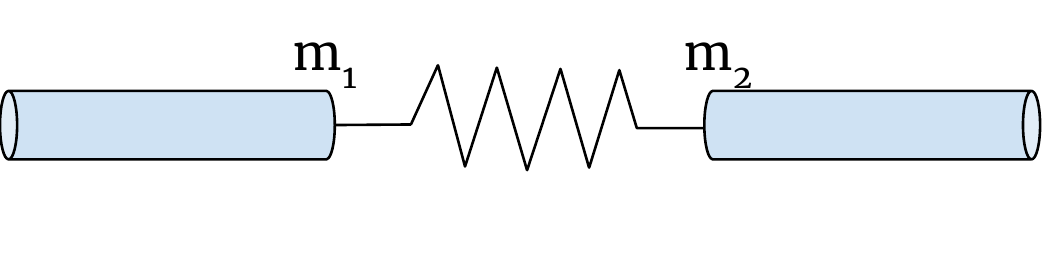}
\caption{Simple Spring Rod Model for Implicit Integration}
\label{fig:simple_spring_rod_model}
\end{figure}
Every spring-rod system can be considered a connection of basis elements: the spring and the rod. Instead of deriving the implicit integration for the whole system, we do for the single element, which is simpler and more straightforward. Considering a spring connecting two rods as shown in Fig.~\ref{fig:simple_spring_rod_model}, the dynamics with implicit integration is
\begin{align*}
    x_{t+1}^{m_1} &= x_{t+1}^R + r_{t+1} \\
    v_{t+1}^{m_1} &= v_{t+1}^R + \omega_{t+1}^R \times r_{t+1} \\
    \Delta x_{t+1} &= x_{t+1}^{m_1} - x_{t+1}^{m_2} \\
    \hat{\Delta x_{t+1}} &= \Delta x_{t+1} / ||\Delta x_{t+1}|| \\
    \Delta v_{t+1} &= v_{t+1}^{m_1} - v_{t+1}^{m_2} \\
    l_{t+1}^{rest} &= (l^{rest} + w_{t+1} * c) * \hat{\Delta x_{t+1}}\\
    v_{t+1}^{proj} &= (\Delta v_{t+1} \cdot \hat{\Delta x_{t+1}}) * \hat{\Delta x_{t+1}} \\
    f_{t+1} &= K (\Delta x_{t+1} - l_{t+1}^{rest}) + k v_{t+1}^{proj} \\
    \tau_{t+1} &= r_{t+1} \times f_{t+1} \\
    v_{t+1}^R &= v_t^R + \frac{f_{t+1}}{m} \Delta t \\
    x_{t+1}^R &= x_{t}^R + v_{t+1}^R \Delta t  \\
    \omega_{t+1}^R &= \omega_t^R + I_{t}^{-1} \tau_{t+1} \Delta t \\
    r_{t+1} &= r_t + \omega_{t+1}^R \Delta t
\end{align*}
where $x^{m_1}$ is the position of $m_1$, $x^R$ is the position of the rod on the left, $r$ is the torque arm from $m_1$ to rod's center of mass, $v^{m_1}$ is the linear velocity of $m_1$, $v^R$ is the rod linear velocity, $\omega^R$ is the rod angular velocity, $\Delta x$ is the spring vector, $||\Delta x||$ is 2-norm of spring vector, i.e. spring length, $\hat{\Delta x}$ is the spring direction, $\Delta v$ is the relative velocity of two ends of the spring, $l^{rest}$ is the spring rest length, $w$ is the motor position, $c$ is the motor position scaler, $v^{proj}$ is the projection of relative velocity $v$ onto the spring direction $\hat{\Delta x}$, $f$ is the spring force under Hooke's law, $K$ is spring stiffness, $k$ is spring damping, $\tau$ is the torque of spring force onto the rod, $m$ is rod mass, and $I^{-1}$ is the inverse of rod inertia matrix. $t, t+1$ mean the current or next time step.

Solving the above 13 equations directly is impossible, since we only have 13 equations but all terms with subscript $t+1$ are unknown, which are way more than 13. Besides, the quadratic terms like $ \omega_{t+1}^R \times r_{t+1}, r_{t+1} \times f_{t+1}$ make the whole system nonlinear.

So we simplify the above equations by replacing part of unknown $t+1$ terms by $t$ terms, which conserve the stable feature of the implicit-integration but much easier to solve. The simplified implicit integration is
\begin{align*}
    x_{t+1}^{m_1} &= x_{t+1}^R + r_{t+1} \\
    v_{t+1}^{m_1} &= v_{t+1}^R + \omega_{t}^R \times r_{t+1} \\
    \Delta x_{t+1} &= x_{t+1}^{m_1} - x_{t}^{m_2} \\
    \hat{\Delta x_{t}} &= \Delta x_{t} / ||\Delta x_{t}|| \\
    \Delta v_{t+1} &= v_{t+1}^{m_1} - v_{t}^{m_2} \\
    l_{t}^{rest} &= (l^{rest} + w_{t} * c) * \hat{\Delta x_{t}}\\
    v_{t+1}^{proj} &= (\Delta v_{t+1} \cdot \hat{\Delta x_{t}}) * \hat{\Delta x_{t}} \\
    f_{t+1} &= K (\Delta x_{t+1} - l_{t}^{rest}) + k v_{t+1}^{proj} \\
    \tau_{t+1} &= r_{t} \times f_{t+1} \\
    v_{t+1}^R &= v_t^R + \frac{f_{t+1}}{m} \Delta t \\
    x_{t+1}^R &= x_{t}^R + v_{t+1}^R \Delta t  \\
    \omega_{t+1}^R &= \omega_t^R + I_{t}^{-1} \tau_{t+1} \Delta t \\
    r_{t+1} &= r_t + \omega_{t+1}^R \Delta t
\end{align*}
in which all quadratic terms are gone and only have less unknown variables.

Now let's convert above equations to format like $Ax=b$,
\begin{align*}
    f_{t+1} &= K (\Delta x_{t+1} - l_{t}^{rest}) + k v_{t+1}^{proj} \\
            &= K (x_{t+1}^{m_1} - x_{t}^{m_2} - l_{t}^{rest}) + k \Delta v_{t+1} \cdot \hat{\Delta x_{t}} * \hat{\Delta x_{t}} \\
            &= K x_{t+1}^{m_1} - K (x_{t}^{m_2} + l_{t}^{rest}) + k (v_{t+1}^{m_1} - v_{t}^{m_2} ) \cdot \hat{\Delta x_{t}} * \hat{\Delta x_{t}} \\
            &= K x_{t+1}^{m_1} - K (x_{t}^{m_2} + l_{t}^{rest}) \\
            & + k 
            \begin{bmatrix}
                \hat{\Delta x_{t}}_x & 0 & 0 \\
                0 & \hat{\Delta x_{t}}_y & 0 \\
                0 & 0 & \hat{\Delta x_{t}}_z
            \end{bmatrix}
            \begin{bmatrix}
                \hat{\Delta x_{t}}_x & \hat{\Delta x_{t}}_y & \hat{\Delta x_{t}}_z \\
                \hat{\Delta x_{t}}_x & \hat{\Delta x_{t}}_y & \hat{\Delta x_{t}}_z \\
                \hat{\Delta x_{t}}_x & \hat{\Delta x_{t}}_y & \hat{\Delta x_{t}}_z 
            \end{bmatrix}
            v_{t+1}^{m_1} \\
            & - k (v_{t}^{m_2}  \cdot \hat{\Delta x_{t}}) * \hat{\Delta x_{t}} \\
\end{align*}
i.e.
\begin{align*}
    &K x_{t+1}^{m_1}  - f_{t+1}\\
    &+ k 
            \begin{bmatrix}
                \hat{\Delta x_{t}}_x & 0 & 0 \\
                0 & \hat{\Delta x_{t}}_y & 0 \\
                0 & 0 & \hat{\Delta x_{t}}_z
            \end{bmatrix}
            \begin{bmatrix}
                \hat{\Delta x_{t}}_x & \hat{\Delta x_{t}}_y & \hat{\Delta x_{t}}_z \\
                \hat{\Delta x_{t}}_x & \hat{\Delta x_{t}}_y & \hat{\Delta x_{t}}_z \\
                \hat{\Delta x_{t}}_x & \hat{\Delta x_{t}}_y & \hat{\Delta x_{t}}_z 
            \end{bmatrix}
            v_{t+1}^{m_1} \\
    & = K (x_{t}^{m_2} + l_{t}^{rest}) + k (v_{t}^{m_2}  \cdot \hat{\Delta x_{t}}) * \hat{\Delta x_{t}} 
\end{align*}

Transform all equations to form $Ax=b$
\begin{align*}
    x_{t+1}^{m_1} - x_{t+1}^R - r_{t+1} &= 0\\
    v_{t+1}^{m_1} - v_{t+1}^R  - [\omega_{t}^R \times] r_{t+1} & = 0\\
        K x_{t+1}^{m_1}  + k 
            \begin{bmatrix}
                \hat{\Delta x_{t}}_x & 0 & 0 \\
                0 & \hat{\Delta x_{t}}_y & 0 \\
                0 & 0 & \hat{\Delta x_{t}}_z
            \end{bmatrix}
            \begin{bmatrix}
                \hat{\Delta x_{t}}_x & \hat{\Delta x_{t}}_y & \hat{\Delta x_{t}}_z \\
                \hat{\Delta x_{t}}_x & \hat{\Delta x_{t}}_y & \hat{\Delta x_{t}}_z \\
                \hat{\Delta x_{t}}_x & \hat{\Delta x_{t}}_y & \hat{\Delta x_{t}}_z 
            \end{bmatrix}
            v_{t+1}^{m_1} &\\
            = f_{t+1} + K (x_{t}^{m_2} + l_{t}^{rest}) + k (v_{t}^{m_2}  \cdot \hat{\Delta x_{t}}) * \hat{\Delta x_{t}} &\\
                v_{t+1}^R - \dfrac{\Delta t} {m} f_{t+1} &= v_t^R\\
    x_{t+1}^R - \Delta t v_{t+1}^R &= x_{t}^R \\
    \omega_{t+1}^R - \Delta t I_{t}^{-1} [r_{t} \times] f_{t+1} &= \omega_t^R  \\
    r_{t+1} - \omega_{t+1}^R \Delta t &= r_t \\
\end{align*}


\section{Gradients at Discontinuities}

Research has shown that time discretization can result in the computation of wrong gradients at collision point~\cite{hu2019difftaichi}. Continuous collision detection (CCD) is a possible solution to circumvent this problem~\cite{hu2019difftaichi,qiao2020scalable}. Nevertheless, CCD is expensive and has many limitations. This work proposed a simple alternative solution that circumvents CCD.

Consider a rigid ball example ~\cite{hu2019difftaichi}, where a rigid ball elastically collides with a friction-less ground, as shown in Figure~\ref{fig:rigid_ball_example}. Lowering the initial ball height will increase the final ball height, since there is
less distance to travel before the ball hits the ground and more after (see the loss curves in Fig.~\ref{fig:rigid_ball_example},  right), i.e., $\partial x_{T} / \partial x_{0} = -1$, where:
\begin{align*}
    \dfrac{\partial x_{T}}{\partial x_{0}} = \prod_{t=0}^{T-1} \dfrac{\partial x_{t+1}}{\partial x_{t}}
\end{align*}
For other time steps before and after collision:
\begin{align*}
    v_{t+1} &= v_{t}; \\
    x_{t+1} &= x_{t} + v_{t+1} \Delta t; \\
    \dfrac{\partial x_{t+1}}{\partial x_t} &= 1.
\end{align*}
Thus, assume the collision happens at $t$, then
\begin{align*}
    v_{t+1} &= -v_t\\
    \dfrac{\partial x_{T}}{\partial x_{0}} &= \prod_{t=0}^{T-1} \dfrac{\partial x_{t+1}}{\partial x_{t}} =  \dfrac{\partial x_{t+1}}{\partial x_{t}}.
\end{align*}
thus in the following sections, we only focus on $ \dfrac{\partial x_{t+1}}{\partial x_{t}}$.
\begin{figure}[ht]
\centering
\includegraphics[height=0.22\linewidth]{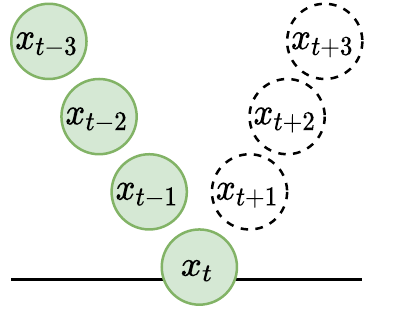}
\includegraphics[trim=10 20 5 15, clip, height=0.24\linewidth]{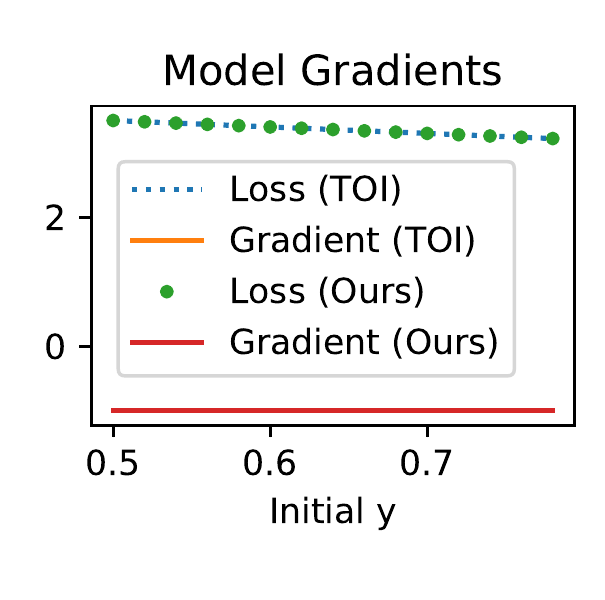}
\caption{left: The rigid ball drops onto the ground at time step $t$ and bound up. right: Our adaptive naive integration (ANI) can get correct gradient and loss as previous time of impact (TOI) integration.} 
\label{fig:rigid_ball_example}
\end{figure}

\subsection{Gradient Analysis for Naive Integration and Time of Impact (TOI) Integration}
The authors~\cite{hu2019difftaichi} mentioned the naive integration would return wrong gradient at collision point and proposed Time of Impact (TOI) integration to circumvent this problem. But no theoretical analysis proposed about why. Here proposed the analysis here.

With naive time integration,
\begin{align*}
    v_{t+1} &= -v_{t} \\
    x_{t+1} &= x_{t} + v_{t+1} \Delta t \\
    \dfrac{\partial x_{t+1}}{\partial x_{t}} &= 1
\end{align*} Thus the gradient has the opposite direction.
However, considering CCD with time of impact (TOI), we have
\begin{align*}
    v_{t+1} &= -v_{t} \\
    TOI &= -x_{t} / v_{t} \\
    x_{t+1} &= x_{t} + v_{t} * TOI + v_{t+1} (\Delta t - TOI) \\
    \dfrac{\partial x_{t+1}}{\partial x_{t}} &= 1 + v_t*(-1/v_{t}) - v_{t} * (-1) * (-1/v_{t}) = -1
\end{align*}
Now the gradient is correct, since TOI introduces additional gradients to $x_t$.

\subsection{Adaptive Naive Integration (ANI) with Correct Gradient}
Instead of CCD and TOI, we proposed an adaptive naive integration (ANI) that can also return correct gradients:
\begin{align*}
    v_{t+1} &= -v_{t} - 2 x_{t}/\Delta t + 2 \widetilde{x_{t}}/\Delta t;\\
    x_{t+1} &= x_{t} + v_{t+1} \Delta t; \\
    \dfrac{\partial x_{t+1}}{\partial x_{t}} &= 1 + \dfrac{\partial v_{t+1}}{\partial x_{t}} \Delta t = 1 -2 = -1,
\end{align*}
where $\widetilde{x_t}$ has the same value to $x_t$ but does not back-propagate gradient. Fig. ~\ref{fig:rigid_ball_example} right shows our ANI can get same correct gradient and loss to TOI.

The proposed contact model can be viewed as a special impulse-based contact model~\cite{Goldstein:2001:CM}:
\begin{align*}
    v_{t+1} &= v_{t} + (-K x_t - k v_t) /m \Delta t \\
    v_{t+1} &= -v_{t} \\
    \dfrac{\partial x_{t+1}}{\partial x_{t}} &= -1 
\end{align*}
Solving these equations, the result is:
\begin{align*}
    K x_t \Delta t / m &= 2 x_t / \Delta t \\
    k v_t \Delta t /m &= -2 \widetilde{x_t}/ \Delta t + 2 v_t
\end{align*}

Consider now a more general case of the impulse-based contact model, $v_{t+1} = - e v_{t}$, where $e \in [0, 1]$.
\begin{theorem}
For impulse-based contact model, the gradient has a correct direction if $K/m > 1/\Delta t^2$ and $k/m = (1+e)/\Delta t + K \widetilde{x_t}/(mv_{t})$, where $e$ is restitution, $K < 0, k < 0$ are constants or functions about $x_t, v_t$ but don't propagate gradient.
\end{theorem}
\begin{proof}
$\ \ \ $
\vspace{-.2in}
\begin{align*}
    v_{t+1} &= v_{t} + (-K x_{t} - k v_t)/m \Delta t;\\
    x_{t+1} &= x_{t} + v_{t+1} \Delta t;\\
    \dfrac{\partial x_{t+1}}{\partial x_{t}} &= 1 - \Delta t^2 K/m
\end{align*}

In order to ensure the gradient has correct direction, the following condition should be satisfied:
\begin{align*}
    \dfrac{\partial x_{t+1}}{\partial x_{t}} &=  1 - \Delta t^2 K/m  < 0
\end{align*}
Thus, $K/m > 1/\Delta t^2$.
Since $v_{t+1} = -ev_{t}$:
\begin{align*}
     -e v_{t} &= v_{t} - K \widetilde{x_{t}}/m \Delta t - k v_{t}/m \Delta t\\
     k/m &= (1+e)/\Delta t + K \widetilde{x_t}/(mv_{t})
\end{align*}
\end{proof}

In the above rigid ball example, we can set $K=2m/\Delta t^2$ to ensure that the gradient is -1.

\subsection{Gradient Analysis with/without Differentiable Collision Checker}
The above analysis has shown that the impulse-based contact model could return correct gradients. Here we discuss whether collision checker is necessary for getting correct gradient direction. The collision checker returns the contact point and intersection distance. Assuming the ground level is $0$ in the rigid ball example, the intersection distance is $0-x_{t}=-x_{t}$. The dynamics at contact point is
\begin{align*}
    v_{t+1} &= v_{t} + (-K x_{t} - k v_{t}) /m \Delta t \\
    x_{t+1} &= x_{t} + v_{t+1} \Delta t
\end{align*}
For the system identification task to identify ground stiffness $K$, $\dfrac{\partial x_{t+1}}{\partial K}$ should be positive since the ball rebounds faster on a stiffer ground.
\begin{align*}
    \dfrac{\partial x_{t+1}}{\partial K} &= - x_{t}\Delta t^2/m > 0 
\end{align*}
which is always true since $x_{t}<0$ at the contact point. Since $x_{t}$ is not a function of $K$, whether the collision checker is differentiable will not affect the gradient correctness.

If there are multiple collisions in a time interval $T$, which probably happen during the recurrent training, will the gradients still have correct direction? Considering the rigid ball example, we assume there are two continues collisions at $t-1, t$. 

\begin{align*}
    v_{t} &= v_{t-1} -K x_{t-1}\Delta t /m -k v_{t-1}\Delta t/m\\
    x_{t} &= x_{t-1} + v_{t} \Delta t \\
    v_{t+1} &= v_{t} -K x_{t} \Delta t/m  -k v_{t} \Delta t/m\\
    x_{t+1} &= x_{t} + v_{t+1} \Delta t    
\end{align*}

If the collision checker is not differentiable, the gradient to $K$ is 
\begin{align*}
    \dfrac{\partial x_{t+1}}{\partial K} &= - \widetilde{x_{t}}\Delta t^2/m > 0 
\end{align*}
where $\widetilde{x_{t}}$ has same value to $x_t$ but doesn't back propagate gradients.

If the collision checker is differentiable, the gradient to $K$ is 
\begin{align*}
    \dfrac{\partial v_t}{\partial K} &= \dfrac{-x_{t-1} \Delta t}{m} > 0\\
    \dfrac{\partial x_{t}}{\partial K} &= - x_{t-1}\Delta t^2/m > 0 \\
    \dfrac{\partial v_{t+1}}{\partial K} &= \dfrac{-x_t \Delta t}{m} \dfrac{\partial x_t}{\partial K} + \dfrac{-k \Delta t}{m} \dfrac{\partial v_t}{\partial K} \\
    \dfrac{\partial x_{t+1}}{\partial K} &= \dfrac{\partial x_{t}}{\partial K} + \dfrac{\partial v_{t+1}}{\partial K} \Delta t \\
    &= (- \dfrac{x_{t}\Delta t^2}{m} + 1 - \dfrac{k\Delta t}{m}) \dfrac{\partial x_{t}}{\partial K} 
\end{align*}
where $x_{t+1}, x_{t} < 0$. Since $\Delta t=0.001s$, term $- \dfrac{x_{t}\Delta t^2}{m} + 1 - \dfrac{k\Delta t}{m}$ could rarely be negative which means wrong gradient direction. To guarantee correct gradient direction, we apply the detaching trick to stop gradient propagation through $v_t$ and $v_{t-1}$, (by detaching from computation graph in PyTorch), which leads to:
\begin{align*}
    v_{t} &= v_{t-1} -K x_{t-1}\Delta t /m -k \widetilde{v_{t-1}}\Delta t/m\\
    x_{t} &= x_{t-1} + v_{t} \Delta t \\
    v_{t+1} &= v_{t} -K x_{t} \Delta t/m  -k \widetilde{v_{t}} \Delta t/m\\
    x_{t+1} &= x_{t} + v_{t+1} \Delta t    
\end{align*}
where $\widetilde{v_{t-1}}, \widetilde{v_{t}}$ have same value to $v_{t-1}, v_t$, but doesn't backpropagate gradient.

Thus we can ignore the damping term $k$ to simplify the following gradient computation, 
\begin{align*}
    v_{t} &= v_{t-1} -K x_{t-1}\Delta t /m \\
    x_{t} &= x_{t-1} + v_{t} \Delta t \\
    v_{t+1} &= v_{t} -K x_{t} \Delta t/m  \\
    x_{t+1} &= x_{t} + v_{t+1} \Delta t    
\end{align*}
If the collision checker is differentiable, the gradient to $K$ is 
\begin{align*}
    \dfrac{\partial x_{t}}{\partial K} &= - x_{t-1}\Delta t^2/m > 0 \\
    \dfrac{\partial x_{t+1}}{\partial K} &=  \dfrac{\partial x_{t}}{\partial K} - \dfrac{x_{t}\Delta t^2}{m} \dfrac{\partial x_{t}}{\partial K} \\
    &= (- \dfrac{x_{t}\Delta t^2}{m} + 1) \dfrac{\partial x_{t}}{\partial K} > 0
\end{align*}
since $x_t < 0, x_{t-1} < 0$. With mathematical induction, we can easily get if there are multiple contacts in the time interval $[i, j]$, 
\begin{align*}
    \dfrac{\partial x_{j+1}}{\partial K} &= (- \dfrac{x_{j}\Delta t^2}{m} + 1) (- \dfrac{x_{j-1}\Delta t^2}{m} + 1) ... (- \dfrac{x_{i+1}\Delta t^2}{m} + 1) \dfrac{ -x_{i} \Delta t^2}{m} > 0
\end{align*}
since $x_i, x_{i+1}, ... , x_{j-1}, x_{j} < 0$. Thus no matter the collision checker is differentiable or not, we can always get correct gradient direction. Gradient clipping is necessary to avoid exploding gradients.

\end{document}